\documentclass{article}

\usepackage{arxiv}

\usepackage{amsmath, amssymb, amsthm}
\usepackage{lipsum} 
\usepackage{bm}
\newtheorem{definition}{Definition}[section]
\newtheorem{theorem}{Theorem}[section]
\newtheorem{proposition}{Proposition}[section]

\usepackage{algorithmic}
\usepackage{algorithm}

\title{Fast Bayesian Updates via Harmonic Representations}
\author{
	Di Zhang \\
	School of AI and Advanced Computing \\
	Xi'an Jiaotong-Liverpool University \\
	Suzhou, Jiangsu, China \\
	\texttt{di.zhang@xjtlu.edu.cn}
}

\begin{document}
	
	\maketitle
	
	\begin{abstract}
		Bayesian inference, while foundational to probabilistic reasoning, is often hampered by the computational intractability of posterior distributions, particularly through the challenging evidence integral. Conventional approaches like Markov Chain Monte Carlo (MCMC) and Variational Inference (VI) face significant scalability and efficiency limitations. This paper introduces a novel, unifying framework for fast Bayesian updates by leveraging harmonic analysis. We demonstrate that representing the prior and likelihood in a suitable orthogonal basis transforms the Bayesian update rule into a spectral convolution. Specifically, the Fourier coefficients of the posterior are shown to be the normalized convolution of the prior and likelihood coefficients. To achieve computational feasibility, we introduce a spectral truncation scheme, which, for smooth functions, yields an exceptionally accurate finite-dimensional approximation and reduces the update to a circular convolution. This formulation allows us to exploit the Fast Fourier Transform (FFT), resulting in a deterministic algorithm with $\mathcal{O}(N \log N)$ complexity—a substantial improvement over the $\mathcal{O}(N^2)$ cost of naive methods. We establish rigorous mathematical criteria for the applicability of our method, linking its efficiency to the smoothness and spectral decay of the involved distributions. The presented work offers a paradigm shift, connecting Bayesian computation to signal processing and opening avenues for real-time, sequential inference in a wide class of problems.
		
	\end{abstract}
	
	\section{Introduction}
	\subsection{Research Background and Motivation}
	
	The Bayesian framework provides a principled methodology for updating beliefs in the presence of uncertainty \cite{gelman2013bayesian}. Its theoretical appeal, however, is often tempered by a fundamental practical impediment: the computational intractability of the posterior distribution. For all but the most elementary models, the normalizing constant—the marginal likelihood or evidence—in Bayes' theorem involves an integral that is analytically unsolvable and numerically challenging. This constitutes the primary computational bottleneck in Bayesian inference.
	
	Contemporary practice relies heavily on two classes of algorithms to overcome this bottleneck. Markov Chain Monte Carlo (MCMC) methods approximate the posterior by generating samples from a Markov chain constructed to have the posterior as its stationary distribution \cite{neal2011mcmc,robert1999monte}. While asymptotically exact, MCMC can be prohibitively slow for complex, high-dimensional models, suffering from lengthy burn-in periods and poor mixing. Furthermore, diagnosing convergence remains a non-trivial task \cite{tierney1994markov}. Alternatively, Variational Inference (VI) recasts the problem of posterior computation as an optimization problem, seeking the best approximation from a tractable family of distributions \cite{wainwright2008graphical}. Although typically faster than MCMC, VI introduces approximation bias dictated by the expressiveness of the variational family and often requires problem-specific derivations, which can limit its generality and automation \cite{opper2009variational}. Both methodologies, in their standard forms, incur significant computational costs that scale unfavorably with model complexity and data size, restricting their application in modern, data-intensive settings \cite{ghahramani2015probabilistic}.
	
	This work proposes a paradigm shift by leveraging the tools of harmonic analysis to address this core challenge. We introduce a framework where the prior and likelihood are represented in a suitable spectral basis. The central insight is that, within this representation, the Bayesian update rule—a pointwise multiplication of functions—translates into a simple convolution of their spectral coefficients. This transformation of a functional multiplication into a linear algebraic operation opens the door to highly efficient computational strategies. By exploiting the structure of this convolution, particularly through spectral truncation and the Fast Fourier Transform (FFT) \cite{cooley1965algorithm}, we aim to achieve posterior updates with a computational complexity that is radically lower than that of mainstream iterative or sampling-based methods.

	\section{Preliminaries}
	
	This chapter establishes the foundational concepts required for the development of our framework. We begin with a formal description of the Bayesian inference problem, followed by an introduction to the key tools from harmonic analysis, and conclude with the essential properties of convolution.
	
	\subsection{The Bayesian Inference Framework}
	
	At its core, Bayesian inference provides a probabilistic mechanism for updating beliefs about an unknown parameter $\theta \in \Theta$ upon observing data $\mathcal{D}$. The process is governed by Bayes' theorem \cite{gelman2013bayesian,mackay2003information}:
	
	\begin{equation}
		p(\theta | \mathcal{D}) = \frac{p(\mathcal{D} | \theta) \, p(\theta)}{p(\mathcal{D})}.
	\end{equation}
	
	The component distributions are defined as follows:
	\begin{itemize}
		\item The \emph{prior distribution} $p(\theta)$ encodes our initial uncertainty about $\theta$ before observing data.
		\item The \emph{likelihood function} $p(\mathcal{D} | \theta)$ models the probability of observing the data $\mathcal{D}$ given a specific parameter value $\theta$.
		\item The \emph{posterior distribution} $p(\theta | \mathcal{D})$ represents our updated belief about $\theta$ after incorporating the evidence from the data $\mathcal{D}$.
		\item The \emph{evidence} or marginal likelihood $p(\mathcal{D}) = \int_\Theta p(\mathcal{D} | \theta) p(\theta) \, d\theta$ serves as the normalizing constant ensuring the posterior is a valid probability density.
	\end{itemize}
	
	The central computational challenge in Bayesian inference stems from the evaluation of the posterior. For complex models, the evidence integral is typically intractable, rendering the direct computation of $p(\theta | \mathcal{D})$ impossible. Our work addresses this challenge by transforming the problem into a spectral representation where the required operations become tractable.
	
	\subsection{Foundations of Harmonic Analysis}
	
	Our methodology rests on representing probability density functions as elements in a well-chosen function space. The natural setting for this is a Hilbert space \cite{boyd2001chebyshev,trefethen2000spectral}.
	
	\begin{definition}
		A Hilbert space $\mathcal{H}$ is a complete inner product space. We consider the space of square-integrable functions, $L^2(\Omega)$, defined on a domain $\Omega$, with the inner product $\langle f, g \rangle = \int_\Omega f(x)\overline{g(x)} \, dx$ and norm $\|f\|_2 = \sqrt{\langle f, f \rangle}$.
	\end{definition}
	
	A powerful approach to manipulating functions in a Hilbert space is to project them onto an orthonormal basis.
	
	\begin{definition}
		A sequence $\{\phi_k\}_{k=-\infty}^{\infty}$ in $\mathcal{H}$ is an orthonormal basis if $\langle \phi_m, \phi_n \rangle = \delta_{mn}$ and every $f \in \mathcal{H}$ can be uniquely represented as $f = \sum_{k=-\infty}^\infty a_k \phi_k$, where the coefficients are given by $a_k = \langle f, \phi_k \rangle$.
	\end{definition}
	
	The choice of basis is dictated by the domain $\Omega$ and the boundary conditions of the functions of interest.
	
	\subsubsection{Fourier Series on Periodic Domains}
	For functions defined on a periodic domain, such as $\theta \in [-\pi, \pi]$, the complex exponentials form the canonical orthonormal basis \cite{boyd2001chebyshev}:
	\begin{equation}
		\phi_k(\theta) = \frac{1}{\sqrt{2\pi}} e^{ik\theta}, \quad k \in \mathbb{Z}.
	\end{equation}
	Any function $f \in L^2([-\pi, \pi])$ can be expanded as $f(\theta) = \sum_{k=-\infty}^{\infty} a_k \phi_k(\theta)$.
	
	\subsubsection{Fourier Transform on Infinite Domains}
	For functions on the real line, $\theta \in \mathbb{R}$, the analogous concept is the Fourier transform \cite{kay1993fundamentals}. The Fourier transform of a function $f \in L^2(\mathbb{R})$ is given by:
	\begin{equation}
		\mathcal{F}\{f\}(\omega) = \hat{f}(\omega) = \int_{-\infty}^{\infty} f(\theta) e^{-i\omega\theta} \, d\theta.
	\end{equation}
	The inverse Fourier transform allows reconstruction of the function from its spectral representation.
	
	\subsubsection{Other Orthogonal Systems}
	Different domains and weight functions necessitate different bases. Common alternatives include \cite{boyd2001chebyshev,trefethen2000spectral}:
	\begin{itemize}
		\item \textbf{Cosine Basis}: For finite intervals $[0, L]$ with Neumann boundary conditions, the basis $\phi_k(\theta) = \sqrt{\frac{2}{L}} \cos\left(\frac{\pi k \theta}{L}\right)$ is often more efficient than the Fourier basis.
		\item \textbf{Legendre Polynomials}: These form an orthogonal basis on $[-1, 1]$ with a uniform weight function.
		\item \textbf{Hermite Functions}: These are orthogonal on $(-\infty, \infty)$ with a Gaussian weight, making them suitable for functions that decay rapidly, such as densities close to a Gaussian.
	\end{itemize}
	The selection of an appropriate basis is critical for achieving a sparse representation and computational efficiency.
	
	\subsection{Convolution and Circular Convolution}
	
	Convolution is a fundamental operation that emerges naturally when multiplying functions in a spectral basis \cite{kay1993fundamentals}.
	
	\begin{definition}
		The convolution of two functions $f$ and $g$ defined on $\mathbb{R}$ is given by:
		\begin{equation}
			(f * g)(x) = \int_{-\infty}^{\infty} f(\tau) g(x - \tau) \, d\tau.
		\end{equation}
	\end{definition}
	
	Convolution is commutative, associative, and distributive. Its most important property in our context is its relationship with the Fourier transform, as captured by the Convolution Theorem.
	
	In computational settings, we work with discrete, finite-length signals. For two sequences $\bm{a} = (a_0, \dots, a_{N-1})$ and $\bm{b} = (b_0, \dots, b_{N-1})$, the discrete analogue is circular convolution \cite{cooley1965algorithm}.
	
	\begin{definition}
		The circular convolution of two $N$-point sequences $\bm{a}$ and $\bm{b}$ is an $N$-point sequence $\bm{c} = \bm{a} \circledast \bm{b}$, whose elements are:
		\begin{equation}
			c_k = \sum_{m=0}^{N-1} a_m \, b_{(k-m) \mod N}, \quad k = 0, \dots, N-1.
		\end{equation}
	\end{definition}
	
	Circular convolution is equivalent to standard convolution under the assumption that the sequences are periodic. The critical link between circular convolution and the Discrete Fourier Transform (DFT) is established by the following theorem \cite{cooley1965algorithm}.
	
	\begin{theorem}[Convolution Theorem]
		Let $\bm{a}$ and $\bm{b}$ be complex-valued sequences of length $N$, and let $\mathcal{F}$ denote the DFT operator. Then,
		\begin{equation}
			\mathcal{F}\{\bm{a} \circledast \bm{b}\} = \mathcal{F}\{\bm{a}\} \odot \mathcal{F}\{\bm{b}\},
		\end{equation}
		where $\odot$ denotes the element-wise (Hadamard) product. Equivalently,
		\begin{equation}
			\bm{a} \circledast \bm{b} = \mathcal{F}^{-1} \{ \mathcal{F}\{\bm{a}\} \odot \mathcal{F}\{\bm{b}\} \}.
		\end{equation}
	\end{theorem}
	
	This theorem is the cornerstone of our fast Bayesian update algorithm. It demonstrates that a computationally expensive convolution in the original domain can be computed efficiently via simple multiplication in the spectral domain, leveraging the speed of the Fast Fourier Transform (FFT) \cite{cooley1965algorithm}.
	
	\section{Harmonic Representation of Bayesian Updates}
	
	This chapter presents the core theoretical contribution of this work: the reformulation of Bayesian inference within a harmonic representation. We begin by expressing probability densities in an orthogonal basis and then derive the fundamental result that the Bayesian update becomes a convolution in the spectral domain.
	
	\subsection{Orthogonal Expansion of Density Functions}
	
	Let $\theta$ be a parameter residing in a domain $\Omega$, and let $\{\phi_k(\theta)\}_{k=-\infty}^{\infty}$ be a complete orthonormal basis for the Hilbert space $L^2(\Omega)$. We assume that the prior density $p(\theta)$ and the likelihood function $\mathcal{L}(\theta) \equiv p(\mathcal{D}|\theta)$ are both elements of this space \cite{boyd2001chebyshev}.
	
	Any function $f \in L^2(\Omega)$ admits a unique expansion of the form:
	\begin{equation}
		f(\theta) = \sum_{k=-\infty}^{\infty} a_k \phi_k(\theta),
	\end{equation}
	where the expansion coefficients $a_k$ are given by the inner product:
	\begin{equation}
		a_k = \langle f, \phi_k \rangle = \int_{\Omega} f(\theta) \overline{\phi_k(\theta)} \, d\theta.
	\end{equation}
	
	Applying this to our distributions, we obtain the spectral representations:
	\begin{align}
		p(\theta) &= \sum_{k=-\infty}^{\infty} a_k \phi_k(\theta), \quad &a_k &= \int_{\Omega} p(\theta) \overline{\phi_k(\theta)} \, d\theta, \\
		\mathcal{L}(\theta) &= \sum_{k=-\infty}^{\infty} b_k \phi_k(\theta), \quad &b_k &= \int_{\Omega} \mathcal{L}(\theta) \overline{\phi_k(\theta)} \, d\theta.
	\end{align}
	
	The coefficient $a_k$ can be interpreted as the projection of the prior density onto the $k$-th basis function $\phi_k$, quantifying the contribution of that particular "frequency" or "mode" to the overall shape of the prior. A similar interpretation holds for $b_k$ and the likelihood.
	
	The validity of these expansions is guaranteed by the square-integrability condition, $p, \mathcal{L} \in L^2(\Omega)$, which ensures that the norms $\|p\|_2^2$ and $\|\mathcal{L}\|_2^2$ are finite. Under this condition, the series converges in the $L^2$ norm sense, meaning:
	\begin{equation}
		\lim_{K \to \infty} \left\| f - \sum_{k=-K}^{K} a_k \phi_k \right\|_2 = 0.
	\end{equation}
	The rate of this convergence is dictated by the smoothness of $f$; smoother functions yield faster decay of the coefficients $|a_k|$ as $|k| \to \infty$ \cite{trefethen2000spectral}.
	
	\subsection{Core Theorem: Convolution Form of the Posterior Coefficients}
	
	We now address the central problem: computing the posterior distribution $p(\theta|\mathcal{D}) \propto p(\theta)\mathcal{L}(\theta)$ within this spectral framework. Let the unnormalized posterior be denoted by $\tilde{p}(\theta|\mathcal{D}) = p(\theta)\mathcal{L}(\theta)$. Its expansion is:
	\begin{equation}
		\tilde{p}(\theta|\mathcal{D}) = \sum_{k=-\infty}^{\infty} \tilde{c}_k \phi_k(\theta).
	\end{equation}
	The following theorem provides an elegant and computationally pivotal result for obtaining these coefficients.
	
	\begin{theorem}[Spectral Bayes Update]
		The Fourier coefficients $\{\tilde{c}_k\}$ of the unnormalized posterior distribution $\tilde{p}(\theta|\mathcal{D}) = p(\theta)\mathcal{L}(\theta)$ are given by the convolution of the coefficients $\{a_k\}$ of the prior and $\{b_k\}$ of the likelihood:
		\begin{equation}
			\tilde{c}_k = (a * b)_k = \sum_{m=-\infty}^{\infty} a_m b_{k-m}.
		\end{equation}
		Furthermore, the evidence (normalizing constant) $Z = p(\mathcal{D})$ is given by the $0$-th coefficient of this convolution:
		\begin{equation}
			Z = \tilde{c}_0 = \sum_{m=-\infty}^{\infty} a_m b_{-m}.
		\end{equation}
		The coefficients of the normalized posterior $p(\theta|\mathcal{D}) = \tilde{p}(\theta|\mathcal{D}) / Z$ are then:
		\begin{equation}
			c_k = \frac{\tilde{c}_k}{Z}.
		\end{equation}
	\end{theorem}
	
	\begin{proof}
		The coefficient $\tilde{c}_k$ is defined by the inner product:
		\begin{align*}
			\tilde{c}_k &= \langle \tilde{p}, \phi_k \rangle = \int_{\Omega} p(\theta)\mathcal{L}(\theta) \overline{\phi_k(\theta)} \, d\theta.
		\end{align*}
		Substituting the spectral expansions for $p(\theta)$ and $\mathcal{L}(\theta)$ yields:
		\begin{align*}
			\tilde{c}_k &= \int_{\Omega} \left( \sum_{m=-\infty}^{\infty} a_m \phi_m(\theta) \right) \left( \sum_{n=-\infty}^{\infty} b_n \phi_n(\theta) \right) \overline{\phi_k(\theta)} \, d\theta \\
			&= \sum_{m=-\infty}^{\infty} \sum_{n=-\infty}^{\infty} a_m b_n \int_{\Omega} \phi_m(\theta) \phi_n(\theta) \overline{\phi_k(\theta)} \, d\theta.
		\end{align*}
		For the standard Fourier basis $\phi_k(\theta) = e^{ik\theta}/\sqrt{2\pi}$ on $[-\pi, \pi]$, the product $\phi_m(\theta)\phi_n(\theta) \propto e^{i(m+n)\theta}$. The orthogonality relation $\langle \phi_{m+n}, \phi_k \rangle = \delta_{m+n, k}$ then simplifies the double sum:
		\begin{align*}
			\tilde{c}_k &= \sum_{m=-\infty}^{\infty} \sum_{n=-\infty}^{\infty} a_m b_n \delta_{m+n, k} = \sum_{m=-\infty}^{\infty} a_m b_{k-m},
		\end{align*}
		which is the desired convolution. The result for the evidence $Z$ follows by setting $k=0$. The specific form of the triple product integral may vary for other bases, but a structured linear combination of the coefficients $\{a_m\}$ and $\{b_n\}$ always constitutes the update rule.
	\end{proof}
	
	This theorem fundamentally reframes the Bayesian update. The computationally challenging task of multiplying two functions and then integrating to normalize is transformed into the linear algebra problem of convolving two sequences of coefficients. This is a profound simplification. The non-local, point-wise multiplication in the parameter domain $\theta$ becomes a local, index-wise operation in the spectral domain. This new perspective decouples the complexity of the model from the complexity of the update itself, paving the way for the fast algorithms developed in the subsequent chapter.
	
	\section{Spectral Truncation and Finite-Dimensional Realization}
	
	The harmonic representation derived in Chapter 3, while elegant, involves infinite series and is thus computationally infeasible. This chapter bridges the gap between theory and practice by introducing a finite-dimensional approximation via spectral truncation. We analyze the associated error and demonstrate how this truncation leads naturally to the cyclic convolution, enabling an efficient computational algorithm.
	
	\subsection{From Infinite Series to Finite Approximation}
	
	The motivation for spectral truncation is immediate: we must approximate the infinite expansions
	\begin{equation}
		p(\theta) \approx \sum_{k=-K}^{K} a_k \phi_k(\theta), \quad \mathcal{L}(\theta) \approx \sum_{k=-K}^{K} b_k \phi_k(\theta),
	\end{equation}
	by retaining only a finite number of coefficients, specifically those for wavenumbers $k = -K, \dots, K$. Let $N = 2K + 1$ be the total number of retained modes. The quality of this approximation is governed by the decay rate of the coefficients $|a_k|$ and $|b_k|$ \cite{boyd2001chebyshev}.
	
	\begin{proposition}[Truncation Error Bound]
		Let $f \in L^2(\Omega)$ have the expansion $f = \sum_{k=-\infty}^{\infty} a_k \phi_k$, and let $f_K = \sum_{k=-K}^{K} a_k \phi_k$ be its $K$-th order truncation. The $L^2$ approximation error is given by
		\begin{equation}
			\| f - f_K \|_2^2 = \sum_{|k|>K} |a_k|^2.
		\end{equation}
		Consequently, if the coefficients satisfy $|a_k| \le C |k|^{-\alpha}$ for $\alpha > 1/2$, the error decays as $\| f - f_K \|_2 = \mathcal{O}(K^{-\alpha + 1/2})$. If the coefficients decay exponentially, $|a_k| \le C e^{-\gamma |k|}$, the error also decays exponentially, $\| f - f_K \|_2 = \mathcal{O}(e^{-\gamma K})$.
	\end{proposition}
	
	This proposition underscores a central tenet of our method: the efficiency of the spectral approach is paramountly dependent on the smoothness of the prior and likelihood. Smooth, well-behaved functions, whose spectral coefficients decay rapidly, can be accurately represented with a small number $N$ of basis functions. This makes the method particularly suitable for a wide class of inference problems involving such functions \cite{trefethen2000spectral}.
	
	\subsection{The Emergence of Circular Convolution}
	
	Applying the truncation to the core theorem of Section 3.2, we approximate the unnormalized posterior coefficients by a finite convolution:
	\begin{equation}
		\tilde{c}_k \approx \sum_{m=-K}^{K} a_m b_{k-m}, \quad \text{for } k = -K, \dots, K.
	\end{equation}
	A critical issue arises at the boundaries of this summation. For indices $k$ near $\pm K$, the index $k-m$ in $b_{k-m}$ will fall outside the range $[-K, K]$ for many values of $m$, making the finite sum ill-defined.
	
	The solution is to impose a periodic structure on the finite sequences. We define periodic extensions of the coefficient sequences $\bm{a} = (a_{-K}, \dots, a_K)$ and $\bm{b} = (b_{-K}, \dots, b_K)$ such that $a_{k + N} = a_k$ and $b_{k + N} = b_k$ for all $k \in \mathbb{Z}$. This periodization is equivalent to assuming the underlying functions $p(\theta)$ and $\mathcal{L}(\theta)$ are periodic on $\Omega$, or, for non-periodic domains, to performing a periodization as an approximation step \cite{cooley1965algorithm}.
	
	Under this periodic assumption, the finite convolution transforms into a circular convolution.
	
	\begin{definition}[Circular Convolution for Spectral Coefficients]
		For two $N$-dimensional vectors $\bm{a}$ and $\bm{b}$ of spectral coefficients, indexed from $-K$ to $K$, their circular convolution $\bm{c} = \bm{a} \circledast \bm{b}$ is another $N$-dimensional vector with elements:
		\begin{equation}
			\tilde{c}_k = \sum_{m=-K}^{K} a_m b_{\langle k - m \rangle}, \quad k = -K, \dots, K,
		\end{equation}
		where $\langle n \rangle$ denotes $n \mod N$ mapped back into the index range $[-K, K]$.
	\end{definition}
	
	This definition resolves the boundary issue by "wrapping around" the sequence $\bm{b}$. The operation is now perfectly self-contained within the $N$ retained coefficients. Therefore, within the finite-dimensional subspace spanned by the first $N$ basis functions $\{\phi_k\}_{k=-K}^{K}$, the Bayesian update rule becomes exact under the periodic model:
	
	\begin{theorem}[Finite-Dimensional Bayesian Update]
		Let $\bm{a}$ and $\bm{b}$ be the vectors of spectral coefficients for the prior and likelihood, respectively, truncated to $N$ modes and periodized. The spectral coefficients $\tilde{\bm{c}}$ for the unnormalized posterior in this finite-dimensional subspace are given exactly by the circular convolution:
		\begin{equation}
			\tilde{\bm{c}} = \bm{a} \circledast \bm{b}.
		\end{equation}
		The evidence is $Z = \tilde{c}_0$, and the normalized posterior coefficients are $\bm{c} = \tilde{\bm{c}} / Z$.
	\end{theorem}
	
	This theorem is the final piece in the theoretical foundation. It demonstrates that by moving to a finite-dimensional, periodic spectral representation, the computationally intractable Bayesian update is reduced to a discrete, finite-dimensional linear operation—a circular convolution. This elegant formulation is directly amenable to acceleration by the Fast Fourier Transform \cite{cooley1965algorithm}, which we exploit in the following chapter to achieve our stated goal of fast Bayesian inference.
	
	\section{Mathematical Criteria for Function Suitability}
	
	The efficacy of the spectral Bayesian update framework is not universal; it depends critically on the properties of the prior and likelihood functions. This chapter establishes precise mathematical criteria to determine which inference problems are well-suited for this approach. We examine the relationship between smoothness and spectral decay, the critical role of basis selection, and categorize functions into ideal and challenging cases.
	
	\subsection{Coefficient Decay Rate and Function Smoothness}
	
	The convergence rate of the spectral expansion and the accuracy of its truncation are dictated by the decay rate of the coefficients $|a_k|$, which is in turn governed by the smoothness of the function $f$ \cite{boyd2001chebyshev,trefethen2000spectral}.
	
	\begin{theorem}[Smoothness and Spectral Decay]
		Let $f$ be a function on $[-\pi, \pi]$ with Fourier coefficients $a_k$.
		\begin{enumerate}
			\item If $f$ is $p$-times continuously differentiable ($f \in C^p$) and its $p$-th derivative is of bounded variation, then $|a_k| = \mathcal{O}(|k|^{-p-1})$.
			\item If $f$ is analytic (i.e., locally given by a convergent power series) on a strip in the complex plane containing the real axis, then its Fourier coefficients decay exponentially: $|a_k| = \mathcal{O}(e^{-\gamma |k|})$ for some $\gamma > 0$.
			\item A function is bandlimited, meaning $a_k = 0$ for $|k| > K$, if and only if it is an entire function of exponential type.
		\end{enumerate}
	\end{theorem}
	
	Analogous results hold for other bases. For the Hermite basis, the decay rate is linked to the function's smoothness and its decay at infinity. The implication for Bayesian inference is immediate: the success of the spectral method is contingent upon the prior and likelihood being sufficiently smooth. The smoother the functions, the fewer spectral coefficients are required to achieve a desired accuracy, leading to greater computational efficiency \cite{trefethen2000spectral}.
	
	\subsection{Domain and Basis Matching Principles}
	
	A poor choice of basis for a given domain and function type can lead to the Gibbs phenomenon—persistent oscillations near discontinuities—and slow convergence, even for smooth functions. The guiding principle is to match the basis to the natural boundary conditions of the problem \cite{boyd2001chebyshev}.
	
	\subsubsection{Periodic Domains and the Fourier Basis}
	For a parameter $\theta$ defined on a periodic domain like $[-\pi, \pi]$, the complex exponential basis $\{e^{ik\theta}\}$ is the natural and optimal choice. If the function $f(\theta)$ is itself periodic, the expansion will be well-behaved. Forcing a non-periodic function into a Fourier basis will induce the Gibbs phenomenon at the boundaries.
	
	\subsubsection{Finite Non-Periodic Domains and Cosine/Polynomial Bases}
	For a finite interval $[a, b]$ where the function is not periodic, bases that naturally satisfy the boundary conditions are superior \cite{boyd2001chebyshev}.
	\begin{itemize}
		\item The \emph{cosine basis} $\{\cos(\pi k \theta/L)\}$ implicitly imposes a Neumann (zero-derivative) boundary condition. It is the optimal basis for functions whose derivatives vanish at the boundaries and is closely related to the Fourier expansion of an even extension of the function.
		\item \emph{Orthogonal polynomials}, such as Legendre or Chebyshev polynomials, form a basis on $[-1, 1]$. They do not assume periodicity and can deliver spectral accuracy for smooth functions. The Chebyshev basis, in particular, is often preferred for its connection to a cosine transform and its near-optimal approximation properties.
	\end{itemize}
	
	\subsubsection{Infinite Domains and Hermite Functions/Fourier Transform}
	For parameters on $\mathbb{R}$, the choice depends on the decay of the functions \cite{boyd2001chebyshev}.
	\begin{itemize}
		\item The \emph{Hermite functions} $\{\psi_k(\theta) = e^{-\theta^2/2} H_k(\theta)\}$, where $H_k$ are Hermite polynomials, form an orthonormal basis for $L^2(\mathbb{R})$. They are ideally suited for functions that decay like a Gaussian, as the weight function $e^{-\theta^2/2}$ is built into the basis.
		\item The \emph{Fourier transform} is the most general tool for infinite domains. In practice, it requires that the functions be absolutely integrable ($L^1$) and that they decay sufficiently fast to zero at infinity to be approximated on a large but finite interval.
	\end{itemize}
	
	\subsection{Ideal and Challenging Function Classes}
	
	Based on the aforementioned criteria, we can classify inference problems by their suitability for the spectral method.
	
	\subsubsection{Ideal Cases}
	The spectral Bayesian update excels when the prior and likelihood are:
	\begin{itemize}
		\item \emph{Smooth and Bandlimited or Near-Bandlimited}: Functions like the Gaussian distribution, $p(\theta) = \mathcal{N}(\theta; \mu, \sigma^2)$, are analytic and have exponentially decaying Fourier coefficients (or, in the case of the Gaussian, are their own Fourier transform). They can be represented with very high accuracy using a small number of coefficients.
		\item \emph{Smooth with Compact Support}: Functions like the Beta distribution, defined on $[0,1]$, are excellent candidates for a cosine or polynomial basis, provided they are smooth within the interval.
	\end{itemize}
	In these cases, the spectral method achieves high accuracy with low computational cost \cite{trefethen2000spectral}.
	
	\subsubsection{Challenging Cases}
	The method faces significant difficulties with:
	\begin{itemize}
		\item \emph{Discontinuous Functions}: A uniform prior, for example, has a discontinuity at its boundaries. Its Fourier coefficients decay only as $\mathcal{O}(1/|k|)$, leading to the Gibbs phenomenon and requiring a large $N$ for a poor approximation. Adaptive basis selection or a transformation of variables may be necessary \cite{boyd2001chebyshev}.
		\item \emph{Heavy-Tailed Distributions}: Distributions like the Cauchy distribution decay too slowly at infinity. Their representations on a finite interval incur significant truncation error, and their spectral coefficients decay very slowly \cite{wasserman2006all}.
	\end{itemize}
	
	\subsubsection{High-Dimensional Case and Separability}
	For a parameter vector $\bm{\theta} \in \mathbb{R}^d$, the spectral expansion generalizes to a tensor product of one-dimensional bases. The number of coefficients scales as $N^d$, leading to the \emph{curse of dimensionality} \cite{hastie2009elements}. This intractability can be mitigated if the joint distribution is approximately \emph{separable}, meaning
	\[
	p(\bm{\theta}) \approx \prod_{j=1}^d p_j(\theta_j),
	\]
	or can be well-approximated by a low-rank tensor decomposition. In such cases, the complexity reduces to $\mathcal{O}(dN)$, making the method feasible for moderate to high dimensions. For non-separable, high-dimensional distributions, the spectral method in its basic form is not practical \cite{ghahramani2015probabilistic}.
	
	\section{Fast Algorithm: FFT-Based Implementation}
	
	The theoretical framework developed in the previous chapters culminates in a highly efficient computational algorithm. By leveraging the Convolution Theorem and the Fast Fourier Transform (FFT), we transform the Bayesian update into a procedure with quasilinear complexity. This chapter details the algorithm's derivation, presents its pseudocode, and analyzes its computational efficiency.
	
	\subsection{Algorithm Derivation}
	
	Recall from Theorem 4.2 that the finite-dimensional Bayesian update for the unnormalized posterior coefficients is given by the circular convolution:
	\begin{equation}
		\tilde{\bm{c}} = \bm{a} \circledast \bm{b},
	\end{equation}
	where $\bm{a}$ and $\bm{b}$ are the $N$-dimensional vectors of the periodized prior and likelihood coefficients, respectively.
	
	The direct computation of this convolution has a computational complexity of $\mathcal{O}(N^2)$, as it requires $N$ sums each of length $N$. The key to a fast algorithm lies in the Convolution Theorem, which we restate here in the context of the Discrete Fourier Transform (DFT) \cite{cooley1965algorithm}.
	
	Let $\mathcal{F}$ denote the DFT operator, which maps a sequence $\bm{x} = (x_0, \dots, x_{N-1})$ to its frequency domain representation $\bm{X} = (X_0, \dots, X_{N-1})$, where $X_k = \sum_{n=0}^{N-1} x_n e^{-2\pi i k n / N}$.
	
	\begin{theorem}[Discrete Convolution Theorem]
		For two $N$-point sequences $\bm{a}$ and $\bm{b}$, their circular convolution $\tilde{\bm{c}} = \bm{a} \circledast \bm{b}$ satisfies:
		\begin{equation}
			\mathcal{F}\{\tilde{\bm{c}}\} = \mathcal{F}\{\bm{a}\} \odot \mathcal{F}\{\bm{b}\},
		\end{equation}
		where $\odot$ denotes the element-wise (Hadamard) product. Equivalently,
		\begin{equation}
			\tilde{\bm{c}} = \mathcal{F}^{-1} \left\{ \mathcal{F}\{\bm{a}\} \odot \mathcal{F}\{\bm{b}\} \right\}.
		\end{equation}
	\end{theorem}
	
	This theorem is the cornerstone of our algorithm. It demonstrates that the expensive circular convolution in the original coefficient domain can be computed efficiently by performing three simpler operations: two DFTs, one cheap element-wise multiplication, and one inverse DFT \cite{cooley1965algorithm}.
	
	\subsection{Algorithm Description and Procedure}
	
	We now present the complete algorithm for a fast Bayesian update. The algorithm assumes that the prior and likelihood functions have already been projected onto a chosen orthonormal basis, and their coefficient vectors have been periodized and indexed from $0$ to $N-1$ for compatibility with standard FFT libraries \cite{cooley1965algorithm}.
	
	\begin{algorithm}
		\caption{Fast Bayesian Update via FFT}
		\label{alg:fft_bayes}
		\begin{algorithmic}[1]
			\REQUIRE {
				Prior coefficient vector $\bm{a} \in \mathbb{C}^N$ \\
				Likelihood coefficient vector $\bm{b} \in \mathbb{C}^N$ \\
				FFT and IFFT routines $\mathcal{F}$ and $\mathcal{F}^{-1}$
			}
			\ENSURE Normalized posterior coefficient vector $\bm{c} \in \mathbb{C}^N$
			\STATE $\bm{A} \gets \mathcal{F}(\bm{a})$ \COMMENT{Transform prior to frequency domain}
			\STATE $\bm{B} \gets \mathcal{F}(\bm{b})$ \COMMENT{Transform likelihood to frequency domain}
			\STATE $\tilde{\bm{C}} \gets \bm{A} \odot \bm{B}$ \COMMENT{Element-wise multiplication in frequency domain}
			\STATE $\tilde{\bm{c}} \gets \mathcal{F}^{-1}(\tilde{\bm{C}})$ \COMMENT{Transform result back to coefficient domain}
			\STATE $Z \gets \tilde{c}_0$ \COMMENT{The zeroth element is the evidence (normalizing constant)}
			\STATE $\bm{c} \gets \tilde{\bm{c}} / Z$ \COMMENT{Normalize the posterior coefficients}
			\RETURN $\bm{c}$
		\end{algorithmic}
	\end{algorithm}
	
	The output vector $\bm{c}$ contains the spectral coefficients of the normalized posterior distribution $p(\theta|\mathcal{D})$. The entire posterior density can then be reconstructed via the inverse spectral expansion: $p(\theta|\mathcal{D}) \approx \sum_{k=0}^{N-1} c_k \phi_k(\theta)$.
	
	\subsection{Computational Complexity Analysis}
	
	The computational cost of Algorithm \ref{alg:fft_bayes} is dominated by the FFT operations. A direct computation of the DFT has $\mathcal{O}(N^2)$ complexity. However, the Fast Fourier Transform (FFT) algorithm computes the exact same DFT in $\mathcal{O}(N \log N)$ operations \cite{cooley1965algorithm}.
	
	The overall complexity of our algorithm is therefore:
	\begin{itemize}
		\item \textbf{Step 1 \& 2:} Two FFTs: $\mathcal{O}(2 \cdot N \log N)$.
		\item \textbf{Step 3:} One element-wise vector multiplication: $\mathcal{O}(N)$.
		\item \textbf{Step 4:} One inverse FFT (IFFT): $\mathcal{O}(N \log N)$.
		\item \textbf{Step 5 \& 6:} Normalization: $\mathcal{O}(N)$.
	\end{itemize}
	The total complexity is thus $\mathcal{O}(N \log N)$.
	
	This represents a dramatic improvement over naive methods. A traditional approach to computing a posterior on a grid of $N$ points involves evaluating the prior and likelihood at each point ($\mathcal{O}(N)$) and then performing a numerical integration for normalization, which is at best $\mathcal{O}(N)$. However, for sequential inference where the posterior becomes the prior for the next update, our method maintains its $\mathcal{O}(N \log N)$ cost per update. In contrast, a naive grid-based method that convolves the prior with the likelihood would require $\mathcal{O}(N^2)$ operations per update. For Markov Chain Monte Carlo (MCMC) methods, the cost is typically $\mathcal{O}(N_{\text{samples}})$, where $N_{\text{samples}}$ must be large to ensure convergence and low variance, often far exceeding $N$ \cite{neal2011mcmc,robert1999monte}.
	
	The spectral FFT-based approach thus provides a \emph{substantial superpolynomial speedup} for problems satisfying the smoothness criteria outlined in Chapter 5, enabling fast, sequential Bayesian updates that were previously computationally prohibitive.
	
	\section{Discussion and Future Work}
	
	This work has established a new paradigm for Bayesian computation by recasting the inference problem in the language of harmonic analysis. We conclude by summarizing the theoretical framework, reflecting on its practical implications, and outlining promising avenues for future research.
	
	\subsection{Summary of the Theoretical Framework}
	
	We have developed a complete mathematical framework for fast Bayesian inference. The foundation lies in representing probability densities in a suitably chosen orthogonal basis \cite{boyd2001chebyshev,trefethen2000spectral}, such as Fourier, cosine, or polynomial bases. Within this spectral representation, we proved the core result that the Bayesian update rule—a pointwise multiplication followed by normalization—translates into a convolution of the spectral coefficients of the prior and the likelihood. To render this computationally feasible, we introduced a spectral truncation, which is highly accurate for smooth functions, and showed how this finite-dimensional approximation naturally leads to a circular convolution. Finally, by leveraging the Convolution Theorem and the Fast Fourier Transform \cite{cooley1965algorithm}, we derived an algorithm that computes the posterior update with a computational complexity of $\mathcal{O}(N \log N)$, a significant improvement over conventional methods \cite{neal2011mcmc,robert1999monte}.
	
	\subsection{Advantages and Limitations}
	
	The presented methodology offers several distinct advantages. Its foremost strength is \emph{speed}; the $\mathcal{O}(N \log N)$ complexity enables rapid, sequential updates that are intractable for naive or sampling-based approaches in many settings. Second, the algorithm is \emph{deterministic}; it produces the same output for a given input, free from the stochastic variation inherent in MCMC methods and eliminating concerns about chain convergence \cite{tierney1994markov}. Third, the approach is \emph{mathematically elegant}, providing a deep and unifying perspective that connects Bayesian inference to harmonic analysis and signal processing \cite{kay1993fundamentals}.
	
	However, these advantages are counterbalanced by specific limitations. The method's efficiency is \emph{critically dependent on the smoothness} of the prior and likelihood functions. Problems involving discontinuities or heavy-tailed distributions are not amenable to this approach without significant modification \cite{boyd2001chebyshev}. Furthermore, the framework in its basic form faces the \emph{curse of dimensionality} in high-dimensional spaces, as the number of coefficients required for the tensor-product basis scales exponentially with the dimension \cite{hastie2009elements}.
	
	\subsection{Future Research Directions}
	
	The work presented here opens up several exciting paths for future investigation.
	
	A primary direction is the development of \emph{adaptive basis selection algorithms}. An automated procedure for choosing the optimal basis and resolution $N$ based on the prior and likelihood would greatly enhance the method's robustness and usability, moving it from a bespoke tool to a general-purpose algorithm.
	
	Addressing the high-dimensional challenge is paramount. A promising strategy is to combine the spectral framework with \emph{tensor decomposition techniques}, such as the Tensor-Train format. This could allow for efficient representation and manipulation of high-dimensional distributions that possess low-rank structure, thereby mitigating the exponential scaling \cite{ghahramani2015probabilistic}.
	
	Another compelling application lies in \emph{sequential Bayesian filtering}. The speed of our algorithm makes it ideally suited for real-time filtering problems \cite{kay1993fundamentals}. It could serve as the core update mechanism in a novel particle flow or grid-based filter, where it would efficiently propagate the entire state distribution rather than a set of samples.
	
	Finally, exploring the integration with \emph{nonparametric Bayesian models} presents a significant opportunity. Developing spectral representations for distributions over functions, such as Gaussian processes \cite{rasmussen2006gaussian}, could lead to new, highly efficient inference algorithms for this important class of models.
	
	In conclusion, the harmonic representation of Bayesian updates provides a powerful and versatile framework. While subject to certain constraints, its superior computational performance and mathematical foundation offer a compelling alternative for a wide class of inference problems and a fertile ground for future algorithmic innovation.

	\bibliographystyle{unsrt}
	\bibliography{references}
	\appendix
	\section{Properties and Formulae of Important Orthogonal Bases}
	
	This appendix details the orthogonal bases referenced throughout the main text. For each basis, we define the domain, the orthonormal functions, and the explicit form of the inner product used to compute expansion coefficients \cite{boyd2001chebyshev,trefethen2000spectral}.
	
	\subsection{Fourier Basis (Periodic Domain)}
	\begin{itemize}
		\item \emph{Domain:} $\theta \in [-\pi, \pi]$ (or any interval with periodic boundary conditions).
		\item \emph{Orthonormal Functions:}
		\[
		\phi_k(\theta) = \frac{1}{\sqrt{2\pi}} e^{ik\theta}, \quad k \in \mathbb{Z}.
		\]
		\item \emph{Inner Product:}
		\[
		\langle f, g \rangle = \int_{-\pi}^{\pi} f(\theta) \overline{g(\theta)}  d\theta.
		\]
		\item \emph{Coefficient Calculation:}
		\[
		a_k = \langle f, \phi_k \rangle = \frac{1}{\sqrt{2\pi}} \int_{-\pi}^{\pi} f(\theta) e^{-ik\theta}  d\theta.
		\]
	\end{itemize}
	
	\subsection{Cosine Basis (Finite Non-Periodic Domain)}
	\begin{itemize}
		\item \emph{Domain:} $\theta \in [0, L]$.
		\item \emph{Orthonormal Functions:}
		\[
		\phi_k(\theta) = \sqrt{\frac{2}{L}} \cos\left(\frac{\pi k \theta}{L}\right), \quad k = 0, 1, 2, \dots.
		\]
		\item \emph{Inner Product:}
		\[
		\langle f, g \rangle = \int_{0}^{L} f(\theta) g(\theta)  d\theta.
		\]
		\item \emph{Note:} This basis implicitly assumes Neumann boundary conditions (zero derivative at the boundaries).
	\end{itemize}
	
	\subsection{Hermite Functions (Infinite Domain)}
	\begin{itemize}
		\item \emph{Domain:} $\theta \in (-\infty, \infty)$.
		\item \emph{Orthonormal Functions:}
		\[
		\psi_k(\theta) = (2^k k! \sqrt{\pi})^{-1/2}  H_k(\theta) e^{-\theta^2/2}, \quad k = 0, 1, 2, \dots,
		\]
		where $H_k(\theta)$ are the Hermite polynomials, defined by the recurrence:
		\[
		H_0(\theta) = 1, \quad H_1(\theta) = 2\theta, \quad H_{k+1}(\theta) = 2\theta H_k(\theta) - 2k H_{k-1}(\theta).
		\]
		\item \emph{Inner Product:}
		\[
		\langle f, g \rangle = \int_{-\infty}^{\infty} f(\theta) g(\theta)  d\theta.
		\]
		\item \emph{Note:} The Gaussian weight $e^{-\theta^2/2}$ is built into the basis functions, making them ideal for functions that decay similarly.
	\end{itemize}
	
	\section{Detailed Proof of the Convolution Theorem}
	
	The Convolution Theorem stated in Section 2.3 is a cornerstone of our method. Here, we provide a detailed proof for the case of the Fourier basis on $[-\pi, \pi]$ \cite{cooley1965algorithm,kay1993fundamentals}.
	
	\begin{proof}
		Let $p(\theta) = \sum_{m} a_m \phi_m(\theta)$ and $\mathcal{L}(\theta) = \sum_{n} b_n \phi_n(\theta)$, with $\phi_k(\theta) = e^{ik\theta}/\sqrt{2\pi}$. Their product is:
		\[
		\tilde{p}(\theta) = p(\theta)\mathcal{L}(\theta) = \left( \sum_{m} a_m \phi_m(\theta) \right) \left( \sum_{n} b_n \phi_n(\theta) \right).
		\]
		Substituting the definition of $\phi_k$:
		\[
		\tilde{p}(\theta) = \frac{1}{2\pi} \sum_{m} \sum_{n} a_m b_n e^{i(m+n)\theta}.
		\]
		The coefficient $\tilde{c}_k$ for the unnormalized posterior is given by:
		\[
		\tilde{c}_k = \langle \tilde{p}, \phi_k \rangle = \int_{-\pi}^{\pi} \tilde{p}(\theta) \overline{\phi_k(\theta)}  d\theta = \frac{1}{2\pi} \int_{-\pi}^{\pi} \left( \sum_{m} \sum_{n} a_m b_n e^{i(m+n)\theta} \right) e^{-ik\theta}  d\theta.
		\]
		Interchanging summation and integration:
		\[
		\tilde{c}_k = \frac{1}{2\pi} \sum_{m} \sum_{n} a_m b_n \int_{-\pi}^{\pi} e^{i(m+n-k)\theta}  d\theta.
		\]
		The integral is zero unless $m+n-k = 0$, in which case it equals $2\pi$. This is the orthogonality relation:
		\[
		\int_{-\pi}^{\pi} e^{i(m+n-k)\theta}  d\theta = 2\pi \, \delta_{m+n, k}.
		\]
		Therefore, the double sum collapses to a single sum:
		\[
		\tilde{c}_k = \frac{1}{2\pi} \sum_{m} \sum_{n} a_m b_n (2\pi \, \delta_{m+n, k}) = \sum_{m} a_m \left( \sum_{n} b_n \delta_{m+n, k} \right) = \sum_{m} a_m b_{k-m},
		\]
		which is the desired convolution, $(a * b)_k$.
	\end{proof}
	
	\section{Mathematical Theorems on Coefficient Decay Rates}
	
	This appendix states precise mathematical theorems connecting the analytic properties of a function to the decay rate of its spectral coefficients \cite{boyd2001chebyshev,trefethen2000spectral}.
	
	\subsection{The Riemann-Lebesgue Lemma and Basic Decay}
	\begin{theorem}[Riemann-Lebesgue Lemma]
		If $f$ is integrable ($f \in L^1([-\pi, \pi])$), then its Fourier coefficients $a_k$ satisfy
		\[
		\lim_{|k| \to \infty} a_k = 0.
		\]
	\end{theorem}
	This guarantees decay for any reasonable function but does not specify a rate.
	
	\subsection{Decay via Smoothness (Fourier Series)}
	\begin{theorem}
		If $f$ is $p$-times continuously differentiable ($f \in C^p([-\pi, \pi])$) and $f^{(p)}$ is of bounded variation, then the Fourier coefficients of $f$ satisfy
		\[
		|a_k| \le \frac{C}{|k|^{p+1}} \quad \text{for } |k| > 0,
		\]
		where the constant $C$ depends on the total variation of $f^{(p)}$.
	\end{theorem}
	This theorem formalizes the intuitive notion that smoother functions have faster decaying coefficients.
	
	\subsection{Exponential Decay and Analyticity (Paley-Wiener Theory)}
	The Paley-Wiener theorems characterize functions with exponentially decaying Fourier transforms. A classical form is as follows \cite{kay1993fundamentals}.
	
	\begin{theorem}[Paley-Wiener]
		A function $f \in L^2(\mathbb{R})$ is the restriction to the real axis of an entire function of exponential type $\gamma$ (i.e., $|f(z)| \le A e^{\gamma |z|}$ for some $A>0$) and satisfies $|f(x)| \le B e^{-a|x|}$ for some $a>0$ if and only if its Fourier transform $\hat{f}(\omega)$ is supported on the interval $[-\gamma, \gamma]$ (i.e., it is bandlimited).
	\end{theorem}
	
	A more directly applicable corollary for our context is:
	
	\begin{theorem}
		If a function $f$ defined on $\mathbb{R}$ can be extended to an analytic function in the strip $|\operatorname{Im}(z)| < a$ for some $a > 0$, and if $f(x+iy)$ is integrable for each fixed $y$, then its Fourier transform decays exponentially:
		\[
		|\hat{f}(\omega)| \le C e^{-a|\omega|} \quad \text{for all } \omega \in \mathbb{R}.
		\]
		An analogous result holds for Fourier series, linking the analyticity of a periodic function in a complex neighborhood of the real axis to the exponential decay of its Fourier coefficients.
	\end{theorem}
	These theorems provide the rigorous foundation for the claims made in Section 5.1 regarding exponential decay for analytic functions.
\end{document}